\documentclass[12pt]{article}

\usepackage{amsthm,amsfonts,amsmath,amssymb,epsfig,color,float,graphicx,verbatim}
\usepackage{multirow}

\newtheorem{theorem}{Theorem}

\newtheorem{lemma}[theorem]{Lemma}
\newtheorem{proposition}[theorem]{Proposition}

\newtheorem{claim}[theorem]{Claim}

\theoremstyle{definition}
\newtheorem{definition}[theorem]{Definition}

\newif\ifhyper\IfFileExists{hyperref.sty}{\hypertrue}{\hyperfalse}
\hypertrue
\ifhyper\usepackage{hyperref}\fi

\usepackage{enumitem}
\usepackage[capitalize]{cleveref} 

\author{
Ilias Diakonikolas\thanks{Supported by NSF Award CCF-1652862 (CAREER) and a Sloan Research Fellowship.}\\
University of Southern California\\
{\tt diakonik@usc.edu}\\
\and
Daniel M. Kane\thanks{Supported by NSF Award CCF-1553288 (CAREER) and a Sloan Research Fellowship.}\\
University of California, San Diego\\
{\tt dakane@cs.ucsd.edu}\\
\and
Alistair Stewart\\ 
University of Southern California\\
{\tt alistais@usc.edu}
}

\newcommand{\vol}{{\mathrm{vol}}}
\newcommand{\R}{\mathbb{R}}
\newcommand{\eps}{\epsilon}
\newcommand{\Z}{\mathbb{Z}}
\newcommand{\dtv}{d_{\mathrm TV}}
\newcommand{\eqdef}{\stackrel{{\mathrm {\footnotesize def}}}{=}}
\newcommand{\opt}{\mathrm{OPT}}
\newcommand{\E}{\mathbb{E}}

\title{Learning Multivariate Log-concave Distributions}

\begin{document}

\maketitle

\begin{abstract}
We study the problem of estimating multivariate log-concave probability density functions.
We prove the first sample complexity upper
bound for learning log-concave densities on $\R^d$, for all $d \geq 1$.
Prior to our work, no upper bound on the sample complexity of this learning problem 
was known for the case of $d>3$.

In more detail, we give an estimator that, for any $d \ge 1$ and $\eps>0$,
draws $\tilde{O}_d \left( (1/\eps)^{(d+5)/2} \right)$ samples from an unknown target log-concave density
on $\R^d$, and outputs a hypothesis that (with high probability) is $\eps$-close to the target, 
in total variation distance. 
Our upper bound on the sample complexity comes close to the known lower bound of
$\Omega_d \left( (1/\eps)^{(d+1)/2} \right)$ for this problem.
\end{abstract}

\section{Introduction}  \label{sec:intro}

\subsection{Background and Motivation} \label{ssec:background}

The estimation of a probability density function based on observed data  is a classical and paradigmatic problem
in statistics~\cite{Pearson} with a rich history (see, e.g.,~\cite{BBBB:72, DG85, Silverman:86,Scott:92,DL:01}).
This inference task is known as {\em density estimation} or {\em distribution learning} and
can be informally described as follows:
Given a set of samples from an unknown distribution $f$ that is believed to belong to (or be well-approximated by)
a given family ${\cal D}$, we want to output a hypothesis distribution $h$ that is a good approximation
to the target distribution $f$. 

The first and arguably most fundamental goal in density estimation is to characterize
the {\em sample complexity} of the problem in the minimax sense, i.e., the number of samples {\em inherently}
required to obtain a desired accuracy (in expectation or with high probability). 
In other words, for a given distribution family $\mathcal{D}$ and desired
accuracy $\eps>0$, we are interested in obtaining an estimator for $\mathcal{D}$ with a sample complexity
upper bound of $N = N(\mathcal{D}, \eps)$,
and an information-theoretic lower bound showing that {\em no} estimator
for $\mathcal{D}$ can achieve accuracy $\eps$
with fewer than $\Omega(N)$ samples.
The sample complexity of this unsupervised learning problem depends 
on the {\em structure} of the underlying family $\mathcal{D}.$
Perhaps surprisingly, while density estimation has been studied for several decades,
the sample complexity of learning is not yet well-understood
for various natural and fundamental distribution families.

We emphasize here that there is no known simple complexity measure of a distribution family $\mathcal{D}$
that {\em characterizes} the sample complexity of learning (an unknown distribution from) $\mathcal{D}$ under the total variation distance.
In contrast, the VC dimension of a concept class plays such a role in the PAC model of
learning Boolean functions (see, e.g.,~\cite{BEH+:89, KearnsVazirani:94}).

It should be noted that the classical information-theoretic quantity
of the metric entropy and its variants (e.g., bracketing entropy)~\cite{VWellner96, DL:01, Tsybakov08}\footnote{Roughly speaking, 
the metric entropy of a distribution family $\mathcal{D}$ is the logarithm of the size of the smallest $\eps$-cover of $\mathcal{D}.$ 
A subset ${\cal D}_{\eps} \subseteq {\cal D}$ in a metric space $({\cal D}, d)$ is said to be an \emph{$\eps$-cover of ${\cal D}$} 
with respect to the metric $d: \mathcal{X}^2 \to \R_+,$ if for every $\mathbf{x} \in {\cal D}$ there exists some $\mathbf{y} \in {\cal D}_{\eps}$ 
such that $d(\mathbf{x}, \mathbf{y}) \leq \eps.$ In this paper, we focus on the total variation distance between distributions.}, 
provide {\em upper bounds} on the sample complexity of distribution learning that are {\em not} tight in general. 
Specifically, such upper bounds are suboptimal -- both quantitatively and
qualitatively -- for various distributions families, see, e.g.,~\cite{DKS15} for a natural example.

There are two main strands of research in distribution learning.
The first one  
concerns the learnability of {\em high-dimensional parametric} distribution families
(e.g., mixtures of Gaussians). The sample complexity of learning parametric families
is typically polynomial in the dimension and the goal is to design computationally efficient algorithms.
The second strand -- which is the focus of this paper -- 
studies the learnability of {\em low-dimensional nonparametric} distribution families under various
assumptions on the shape of the underlying density. There has been a long line of work on this strand
within statistics since the 1950's and, more recently, in theoretical computer science 
(see Section~\ref{sec:related} for an overview of related work).
The majority of this literature has studied the univariate (one-dimensional) setting which is by now fairly well-understood
for a wide range of distributions. On the other hand, the {\em multivariate} setting and specifically 
the regime of {\em fixed dimension} is significantly more challenging and 
poorly understood for many natural distribution families. 

\subsection{Our Results and Comparison to Prior Work}  \label{ssec:results}

In this work, we study the problem of density estimation for the family of log-concave distributions on $\R^d$. A distribution on $\R^d$
is log-concave if the logarithm of its probability density function is a concave function (see Definition~\ref{def:lc}).
Log-concave distributions constitute a rich and attractive non-parametric family that is
particularly appealing for modeling and inference~\cite{Walther09}.
They encompass a range of interesting and well-studied distributions, including
uniform, normal, exponential, logistic, extreme value,
Laplace, Weibull, Gamma, Chi and Chi-Squared, and Beta distributions (see, e.g.,~\cite{BagnoliBergstrom05}).
Log-concave distributions have been studied in a range of different contexts including
economics \cite{An:95}, statistics and probability theory (see~\cite{SW14-survey} for a recent survey),
theoretical computer science~\cite{LV07}, and algebra, combinatorics and geometry~\cite{Stanley:89}.

The problem of density estimation for log-concave distributions is of central importance
in the area of non-parametric shape constrained inference. As such, 
this problem has received significant attention in the statistics literature,
see~\cite{Cule10a, DumbgenRufibach:09, DossW16, ChenSam13, KimSam16, BalDoss14, HW16} and references therein,
and, more recently, in theoretical computer science~\cite{CDSS13, 
CDSS14, ADLS17, AcharyaDK15, CanonneDGR16, DKS16-proper-lc}.
In Section~\ref{sec:related}, we provide a detailed summary of related work. 
In this subsection, we confine ourselves to describing the prior work
that is most relevant to the results of this paper.

We study the following fundamental question: 
\begin{center}
{\em How many samples are information-theoretically required
to learn an arbitrary \\ log-concave density on $\R^d$, up to total variation distance $\eps$?}
\end{center}
Despite significant amount of work on log-concave density estimation, 
our understanding of this question even for constant dimension $d$ remains surprisingly poor. 
The only prior work that addresses the $d>1$ case in the finite sample regime is~\cite{KimSam16}.
Specifically,~\cite{KimSam16} study this estimation problem 
with respect to the squared Hellinger distance and obtain the following results: 
\begin{itemize}
\item[(1)] an information-theoretic sample complexity lower bound of
$\Omega_d \left( (1/\eps)^{(d+1)/2} \right)$ for any $d \in \Z_+$, and 
\item[(2)] a sample complexity upper bound
that is tight (up to logarithmic factors) for $d \leq 3$. 
\end{itemize}
{\em Specifically, prior to our work, no finite sample complexity upper bound was known even for $d=4$.}

\medskip

In this paper, we obtain a sample complexity upper bound of $\tilde{O}_d \left( (1/\eps)^{(d+5)/2} \right)$,
for any $d \in \Z_+$, under the total variation distance.
By using the known relation between the total variation and squared Hellinger distances,
our sample complexity upper bound immediately implies the same upper bound under the squared Hellinger distance.
Moreover, the aforementioned lower bound of~\cite{KimSam16} also directly applies
to the total variation distance. Hence, our upper bound is tight up to an $\tilde{O}_d (\eps^{-2})$ multiplicative factor.

To formally state our results, we will need some terminology.

\medskip

\noindent {\bf Notation and Definitions.} Let $f: \R^d \to \R$ be a Lebesgue measurable function.
We will use $f(A)$ to denote $\int_{A} f(x) dx$.
A Lebesgue measurable function $f: \R^d \to \R$ is a probability density function (pdf)
if $f(x) \geq 0$ for all $x \in \R^d$ and  $\int_{\R^d} f(x) dx = 1$.
The {\em total variation distance} between two non-negative measures
$f, g: \R^d \to \R$ is defined as $\dtv(f, g) = \sup_{S} |f(S) - g(S)|$, where
the supremum is over all Lebesgue measurable subsets of the domain.
If $f, g: \R^d \to \R_+$ are probability density functions, then we have that
$\dtv\left(f, g \right) = (1/2) \cdot \| f -g  \|_1 = (1/2) \cdot \int_{\R^d} |f(x) - g(x)| dx.$


\begin{definition} \label{def:lc}
A probability density function $f : \R^d \to \R_+$, $d \in \Z_+$, is called {\em log-concave}
if there exists an upper semi-continuous concave function $\phi: \R^d \to [-\infty, \infty)$
such that $f(x) = e^{\phi(x)}$ for all $x \in \R^d$.
We will denote by $\mathcal{F}_d$ the set of upper semi-continuous,
log-concave densities with respect to the Lebesgue
measure on $\R^d$.
\end{definition}

We use the following definition of learning under the total variation distance.
We remark that our learning model incorporates adversarial model misspecification, and our proposed
estimators are robust in this sense.

\begin{definition}[Agnostic Distribution Learning]  \label{def:learning}
Let ${\cal D}$ be a family of probability density functions on $\R^d$.
A randomized algorithm $A^{\cal D}$ is an {\em agnostic distribution learning algorithm for $\cal D,$}
if for any $\eps>0,$ and any probability density function $f: \R^d \to \R_+$, on input $\eps$ and sample access to $f$,
with probability $9/10,$ algorithm $A^{\cal D}$
outputs a hypothesis density $h: \R^d \to \R_+$ such that
$\dtv(h, f) \leq O(\opt) + \eps$,
where $\opt \eqdef \inf_{g \in \mathcal{D}} \dtv(f, g)$.
\end{definition}

Our agnostic learning definition subsumes Huber's $\eps$-contamination model~\cite{Huber64}, 
which prescribes that the target distribution $f$ is of the form $(1-\eps) g + \eps r$, 
where $g \in {\cal D}$ and $r$ is some arbitrary distribution. 
The main result of this paper is the following theorem:

\begin{theorem}[Main Result] \label{thm:main}
There exists an agnostic learning algorithm for the family $\mathcal{F}_d$ of log-concave densities on $\R^d$ with the following
performance guarantee: For any $d \in \Z_+$, $\eps>0$, and any probability density function $f: \R^d \to \R_+$,
the algorithm draws $O(d/\eps)^{(d+5)/2} \log^2(1/\eps)$ samples from $f$ and, with probability at least $9/10$,
outputs a hypothesis  density $h: \R^d \to \R_+$ such that
$\dtv(h, f) \leq 3 \cdot \opt + \eps$, where $\opt \eqdef \inf_{g \in \mathcal{F}_d} \dtv(f, g)$.
\end{theorem}

To the best of our knowledge, our estimator provides the {\em first} finite sample complexity guarantees for $\mathcal{F}_d$
for any $d>3$. With the exception of~\cite{KimSam16},
prior work on this problem that provides finite sample guarantees has been confined
to the $d=1$ case. As previously mentioned,~\cite{KimSam16} study the case of
general dimension $d$ focusing on the squared Hellinger distance.
Recall that the squared Hellinger distance is defined as
$h^2(f, g) \eqdef \int_{\R^d} (f^{1/2} - g^{1/2})^2 dx$
and that for any two densities $f, g$ it holds $h^2(f, g)  \leq \dtv(f, g) \leq h(f, g)$.
Therefore, the sample lower bound of~\cite{KimSam16} also holds under the total variation distance,
and our sample upper bound immediately applies under the squared Hellinger distance.
This implies that our upper bound is tight up to an $\tilde{O}_d (\eps^{-2})$ multiplicative factor.

Our proposed estimator establishing Theorem~\ref{thm:main} 
is robust to model misspecification with respect to the total variation distance. 
It should be noted that our estimator does not rely on maximum likelihood, as opposed to most of the statistics
literature on this problem. In contrast, our estimator relies on the VC inequality~\cite{VapnikChervonenkis:71, DL:01},
a classical result in empirical process theory (see Theorem~\ref{thm:vc-inequality}).
The VC inequality has been recently used~\cite{CDSS13, CDSS14, ADLS17}
to obtain sharp learning upper bounds for a wide range of {\em one-dimensional} distribution families,
including univariate log-concave densities. As far as we know, ours is the first use of the VC inequality
to obtain learning upper bounds for structured distributions in multiple dimensions. 

\medskip

\noindent {\bf Remark.} Despite its many desirable properties, the maximum likelihood estimator (MLE)
is known to be {\em non-robust} in Huber's contamination 
model\footnote{For log-concave densities, the MLE is known to be robust {\em in the limit} under {\em weaker} metrics~\cite{Dumbgen2011}.}. 
To address this downside, recent work in theoretical computer science~\cite{CDSS14, ADLS17} 
and statistics~\cite{Baraud16} has proposed alternative robust estimators. Moreover, for $4$-dimensional log-concave densities,
it has been conjectured (see, e.g.,~\cite{Wellner15}) that the MLE has suboptimal sample complexity {\em even without noise}.
These facts together provide strong motivation for the design and analysis of 
surrogate estimators with desirable properties, as we do in this work.




\subsection{Related Work} \label{sec:related}

The area of nonparametric density estimation under shape constraints
is a classical topic in statistics starting with the pioneering work of Grenander~\cite{Grenander:56}
on monotone distributions (see~\cite{BBBB:72} for an early and~\cite{GJ:14} for a recent book on the topic).
Various structural restrictions have been studied in the literature, starting from
monotonicity, unimodality, and concavity~\cite{Grenander:56, Brunk:58, PrakasaRao:69, Wegman:70, HansonP:76, Groeneboom:85, Birge:87, Birge:87b,
Fougeres:97, JW:09}. While the majority of the literature has focused on the univariate setting,
a number of works have studied nonparametric distribution families in fixed dimension, see, e.g.,~\cite{BD03, Seregin2010, KoenkerM:10aos}.

In recent years, there has been a body of work in computer science on nonparametric density estimation of
with a focus on both sample and computational 
efficiency~\cite{DDS12soda, DDS12stoc, DDOST13focs, CDSS13, CDSS14, CDSS14b, 
ADHLS15, ADLS17, DHS15, DKS15, DKS15b, DDKT15, DKS16, VV16}.

During the past decade, density estimation of log-concave densities has been extensively investigated.
A line of work in statistics~\cite{Cule10a, DumbgenRufibach:09, DossW16, ChenSam13, BalDoss14} has obtained
a complete understanding of the global consistency properties of the maximum likelihood estimator (MLE) for any dimension $d$.
In terms of finite sample bounds, the sample complexity of log-concave density estimation has been characterized for $d=1$, 
e.g., it is $\Theta(\eps^{-5/2})$ under the variation distance~\cite{DL:01}. 
Moreover, it is known~\cite{KimSam16, HW16} that the MLE is sample-efficient in the univariate setting.
For dimension $d >1$,~\cite{KimSam16} show that the MLE is nearly-sample optimal under the squared Hellinger distance
for $d \leq 3$, and also prove bracketing entropy lower bounds suggesting that the MLE may be sub-optimal for $d > 3$.

A recent line of work in theoretical computer science~\cite{CDSS13, CDSS14, ADLS17, AcharyaDK15, CanonneDGR16, DKS16-proper-lc}
studies the $d=1$ case and obtains sample and computationally efficient estimators under the total variation distance.
Specifically,~\cite{CDSS14, ADLS17} gave sample-optimal robust
estimators for log-concave distributions (among others) based on the VC inequality.

\subsection{Technical Overview} \label{sec:techniques}

In this subsection,
we provide a high-level overview of our techniques establishing Theorem~\ref{thm:main}.
Our approach is inspired by the framework introduced in~\cite{CDSS13, CDSS14}.
Given a family of structured distributions $\mathcal{D}$ that we want to learn,
we proceed as follows: We find an ``appropriately structured'' distribution
family $\mathcal{C}$ that approximates $\mathcal{D}$, in the sense that every density
in  $\mathcal{D}$ is $\eps$-close, in total variation distance, to a density in $\mathcal{C}$.
By choosing the family $\mathcal{C}$ appropriately, we can obtain (nearly-)tight
sample upper bounds for $\mathcal{D}$ from sample upper bounds for $\mathcal{C}$.
Our estimator to achieve this goal (see Lemma~\ref{lem:generic-vc}) leverages the VC inequality.

The aforementioned approach was used in~\cite{CDSS13, CDSS14, ADLS17}
to obtain sample-optimal (and computationally efficient) 
estimators for various {\em one-dimensional} structured distribution families.
In particular, for the family $\mathcal{F}_1$ of univariate log-concave densities,
~\cite{CDSS13} chooses $\mathcal{C}$ to be the family of densities that are piecewise-constant
with $\tilde{O}(1/\eps)$ interval pieces. Similarly,~\cite{CDSS14, ADLS17}
take $\mathcal{C}$ to be the family of densities that are piecewise linear with $\tilde{O}(\eps^{-1/2})$ interval pieces.

Our structural approximation result for the multivariate case can be viewed as an appropriate generalization of the above one-dimensional
results. Specifically, we show that any log-concave density $f$ on $\R^d$ can be $\eps$-approximated,
in total variation distance, by a function $g$ that is essentially defined by
$\tilde{O}_d \left( (1/\eps)^{(d+1)/2} \right)$ hyperplanes.
Once such an approximation has been established, roughly speaking, we exploit the fact that
families of sets defined by a small number of hyperplanes have small VC dimension.
This allows us to use the VC inequality to learn an approximation to $g$
(and, thus, an approximation to $f$) from an appropriate number of samples.
If $V$ is an upper bound on the VC dimension of the resulting set system, the number of samples needed
for this learning task will be $O(V/\eps^2)$.

To prove our structural approximation result for log-concave densities $f$ on $\R^d$ we proceed as follows:
First, we make use of concentration results for log-concave densities implying
that a negligible fraction of $f$'s probability mass comes from points at which $f$
is much smaller than its maximum value. This will allow us to approximate $f$ by a function
$h$ that takes only $\tilde O_d(1/\eps)$ distinct values. Furthermore, the superlevel
sets $h^{-1}([x,\infty))$ will be given by the corresponding superlevel sets for $f$, which are convex.
We then use results from convex geometry to approximate each of these convex sets
(with respect to volume) by inscribed polytopes
with $O_d \left( (1/\eps)^{(d-1)/2} \right)$ facets.
Applying this approximation to each superlevel set of $h$ gives us our function $g$.

We note that a number of constructions are possible here that differ in
exactly how the layers are constructed and what to do when they do or
do not overlap. Many of these constructions are either incorrect or difficult to analyze.
In this work, we provide a simple construction with a succinct proof
that yields a near-optimal sample complexity upper bound.  We believe that a more careful
structural approximation result may lead to the tight sample upper bound, 
and we leave this as an interesting question for future work.

\subsection{Organization}
In Section~\ref{sec:prelims}, we record the basic probabilistic and analytic ingredients we will require.
In Section~\ref{sec:result}, we prove our main result.
Finally, we conclude with a few open problems in Section~\ref{sec:concl}.

\section{Preliminaries} \label{sec:prelims}

\noindent {\bf The VC inequality.}
For $n \in \Z_+$, we will denote $[n] \eqdef \{1,\dots,n\}$.
Let $f: \R^d \to \R$ be a Lebesgue measurable function.
Given a family $\mathcal A$ of measurable subsets of $\R^d$,
we define the {\em $\mathcal{A}$-norm of $f$} by
$\| f \|_{\mathcal A} \eqdef \sup_{A\in \mathcal A} |f(A)| \;.$
We say that a set $X \subseteq \R^d$ is shattered by
$\mathcal A$ if for every $Y \subseteq X$ there exists $A\in\mathcal A$ that satisfies
$A\cap X = Y$. The \emph{VC dimension} of a family of sets $\mathcal A$ over $\R^d$
is defined to be the maximum cardinality of a subset
$X\subseteq \R^d$ that is shattered by $\mathcal A$.
If there is a shattered subset of size $s$ for all $s \in \Z_+$, then we say
that the VC dimension of ${\cal A}$ is $\infty$.

Let $f: \R^d \to \R_+$ be a probability density function. The empirical distribution $\widehat{f}_n$,
corresponding to $n$ independent samples $X_1, \ldots, X_n$ drawn from $f$, is the probability measure
defined by
$\widehat{f}_n(A) = (1/n) \cdot \sum_{i=1}^n \mathbf{1}_{X_i \in A} \;,$
for all $A \subseteq \R^d.$
The well-known \emph{Vapnik-Chervonenkis (VC) inequality} states
the following:

\begin{theorem}[VC inequality, {\cite[p.31]{DL:01}}] \label{thm:vc-inequality}
Let $f: \R^d \to \R_+$ be a probability density function and
$\widehat{f}_n$ be the empirical distribution obtained after drawing $n$ samples from $f$.
Let $\mathcal A $ be a family of subsets over $\R^d$ with VC dimension $V$.
Then,
$\E[ \|f - \widehat{f}_n\|_{\mathcal A}] \leq C \sqrt{V/n} \;,$
where $C$ is a universal constant.
\end{theorem}

\noindent {\bf Approximation of Convex Sets by Polytopes.}
There is a large literature on approximating convex sets by polytopes
(see, e.g., the surveys~\cite{Gruber1993, Bronstein2008}).
We will make essential use of the following theorem that
provides a volume approximation by an inscribed
polytope with a bounded number of facets:

\begin{theorem}[\cite{GMR94, GMR95}] \label{thm:convex-approximation}
Let $d \in \Z_+$.
For any convex body $K \subseteq \R^d$, and $n$ sufficiently large,
there exists a convex polytope $P \subseteq K$ with at most $n$ facets
such that $\vol\left(K \setminus P\right) \leq \frac{Cd}{n^{2/(d-1)}} \vol(K)$, where $C>0$ is a universal constant.
\end{theorem}

\section{Proof of Theorem~\ref{thm:main}} \label{sec:result}
To prove our theorem, we will make essential use of the following general lemma,
establishing the existence of a sample-efficient estimator using the VC inequality:

\begin{lemma}[\cite{CDSS14}] \label{lem:generic-vc}
Let $\mathcal{D}$ be a family of probability density functions over $\R^d$.
Suppose there exists a family $\mathcal{A}$ of subsets of $\R^d$ with VC-dimension $V$ such that
the following holds: For any pair of densities $f_1, f_2 \in \mathcal{D}$ we have that
$ \dtv(f_1, f_2) \leq  \|f_1-f_2\|_{\mathcal{A}} + \eps/2.$
Then, there exists an agnostic learning algorithm for $\mathcal{D}$
with error guarantee $3 \cdot \opt+\eps$ that succeeds with probability $9/10$
using $O(V/\eps^2)$ samples.
\end{lemma}
\begin{proof}
This lemma is implicit in~\cite{CDSS14},
and we include a proof here for completeness. The estimator is extremely simple
and its correctness relies on the VC inequality:
\begin{enumerate}
\item[(1)] Draw $n =O(V/\eps^2)$ samples from $f$;

\item[(2)] Output the density $h \in \mathcal{D}$ that minimizes\footnote{It is straightforward
that it suffices to solve this optimization problem up to an additive $O(\eps)$ error.} the objective function $\|g-\widehat{f}_n\|_{\mathcal{A}}$
over $g \in \mathcal{D}$.
\end{enumerate}
We now show that the above estimator is an agnostic learning algorithm for $\mathcal{D}$.
Let $f^{\ast} = \mathrm{argmin} \{ \dtv(f, g) \mid g \in \mathcal{D} \}$, i.e., $\opt = \dtv(f, f^{\ast})$.
Note that for any pair of densities $f_1, f_2$ and any collection of subsets $\mathcal{A}$
we have that $ \|f_1 - f_2\|_{\mathcal{A}} \leq \dtv(f_1, f_2)$. By Theorem~\ref{thm:vc-inequality} and Markov's inequality,
it follows that with probability at least $9/10$ over the samples drawn from $f$ we have that
$$ \|f-\widehat{f}_n\|_{\mathcal{A}}  \leq \eps/4 \;.$$
Conditioning on this event, we have that
\begin{align*}
\dtv(h, f) & \leq \dtv(f, f^{\ast}) + \dtv(f^{\ast}, h)  \\
&\leq \opt +  \|f^{\ast}-h\|_{\mathcal{A}} + \eps/2  \tag{since $f^{\ast}, h \in \mathcal{D}$} \\
&\leq \opt +  \|f^{\ast}-\widehat{f}_n\|_{\mathcal{A}} + \|h-\widehat{f}_n\|_{\mathcal{A}} +  \eps/2   \\
&\leq \opt +  2 \cdot \|f^{\ast}-\widehat{f}_n\|_{\mathcal{A}} +  \eps/2 \tag{since $\|h-\widehat{f}_n\|_{\mathcal{A}} \leq \|f^{\ast}-\widehat{f}_n\|_{\mathcal{A}}$}\\
&\leq \opt +  2 \cdot \|f^{\ast}-f\|_{\mathcal{A}} + 2 \cdot  \|f-\widehat{f}_n\|_{\mathcal{A}}  +  \eps/2 \\
&\leq \opt +  2 \cdot \dtv(f^{\ast}, f) + 2 \cdot  \|f-\widehat{f}_n\|_{\mathcal{A}}  +  \eps/2 \\
&\leq \opt +  2 \cdot \opt + 2 \cdot \eps/4  +  \eps/2\\
& = 3 \opt+\eps \;.
\end{align*}
This completes the proof of the lemma.
\end{proof}

In view of Lemma~\ref{lem:generic-vc}, to prove Theorem~\ref{thm:main}
we establish the following:

\begin{proposition}\label{prop:main}
There exists a family $\mathcal{A}$ of sets in $\R^d$
whose VC dimension is at most
$V = O(d/\eps)^{(d+1)/2}\log^2(1/\eps)$ that satisfies the condition of Lemma~\ref{lem:generic-vc} for  $\mathcal{F}_d$,
i.e., for any pair of densities $f_1, f_2 \in \mathcal{F}_d$ it holds that
$\dtv(f_1, f_2) \leq  \|f_1-f_2\|_{\mathcal{A}}+\eps/2.$
\end{proposition}

Lemma~\ref{lem:generic-vc} and Proposition~\ref{prop:main} together imply that there exists
an agnostic learner for  $\mathcal{F}_d$ with sample complexity
$$O(V/\eps^2) = O\left(d/\eps\right)^{(d+5)/2} \log^2(1/\eps)  \;,$$
which gives Theorem~\ref{thm:main}.
The main part of this section is devoted to the proof of Proposition~\ref{prop:main}.

\medskip

\noindent {\bf \em Proof Overview:}
The proof has two main steps.
In the first step, we define an appropriately structured
family of functions $\mathcal{C}_{d, \eps}$ so that an arbitrary log-concave density
$f \in \mathcal{F}_d$ can be $\eps$-approximated by a function $g \in \mathcal{C}_{d, \eps}$.
More specifically, each function $g \in \mathcal{C}_{d, \eps}$ takes at most $L = O_d((1/\eps) \log(1/\eps))$ distinct values,
and for each $y \geq 0$, the  sets $g^{-1}([y, \infty))$ are a union of intersections of $H = O_d(\eps^{-(d-1)/2})$ many halfspaces.
We then produce a family $\mathcal{A}_{d,\eps}$ of sets so that for $f,g\in \mathcal{C}_{d, \eps}$, $\dtv(f,g)=\|f-g\|_{\mathcal{A}_{d, \eps}}$ and so that the
VC dimension of $\mathcal{A}_{d, \eps}$ is $\tilde O( d \cdot L \cdot H)$, which yields the desired result.
We proceed with the details below.

\medskip
\noindent {\bf \em Proof of Proposition~\ref{prop:main}:}
We start by formally defining the family of functions  $\mathcal{C}_{d, \eps}$:
\begin{definition}
Given $\eps >0,$ let $\mathcal{C}_{d, \eps}$ be the set of all functions $g: \R^d \to \R$ of the following form:
\begin{itemize}
\item We set $L = L(d, \eps) \eqdef \Theta((d/\eps) \log(d/\eps)).$

\item For $i \in [L]$, let $y_i >0$ and $P_i$ be an intersection of $H \eqdef \Theta(d/\eps)^{(d-1)/2} $ halfspaces in $\R^d$.

\item Given $\{(y_i, P_i)\}_{i=1}^L$, we
define the function $g$ by
\begin{equation} \label{eqn:g}
g(x) =\left\{
	\begin{array}{ll}
		 \max \left\{ y_i \mid i \in [L]: x \in P_i \right\}  & \mbox{if } x \in \cup_{j=1}^L P_j \\
		0 & \mbox{if } x \notin \cup_{j=1}^L P_j \;.
	\end{array}
\right.
\end{equation}
\end{itemize}
Furthermore, we assume that the asymptotic constants used in defining $L$ and $H$ are sufficiently large.
\end{definition}


We are now ready to state and prove our first important lemma:
\begin{lemma} \label{lem:approx}
For any $f \in \mathcal{F}_d$, and any $\eps>0$, there exists $g \in \mathcal{C}_{d, \eps}$ so that
$\|f-g\|_1 = O(\eps)$.
\end{lemma}
\begin{proof}
For $y \in \R_+$ and a function $f: \R^d \to \R_+$
we will denote by
$$L_f(y) \eqdef \{x \in \R^d \mid f(x) \geq y\}$$
its superlevel sets.
We note that, since $f$ is log-concave, $L_f(y)$ is a convex set for all $y \in \R_+$.

We define the desired approximation in a natural way,
by constructing appropriate polyhedral approximations to the superlevel sets
$L_f(y)$ for a finite set of $y$'s in a geometric series with ratio $(1+\eps)$.
Concretely, given $f \in \mathcal{F}_d$ and $\eps>0$,
we define the function $g \in \mathcal{C}_{d, \eps}$ as follows:
For $i \in [L]$, we set $y_i \eqdef M_f \cdot (1-\eps)^i$, where
$M_f$ will denote the maximum value of $f$.
We then consider the collection of convex sets $L_f(y_i)$, $i \in [L]$, and apply
Theorem~\ref{thm:convex-approximation} to approximate each such set by a polytope
with an appropriate number of facets. For each $i \in [L]$, Theorem~\ref{thm:convex-approximation},
applied for $n =O(d/\eps)^{(d-1)/2}$,
prescribes that there exists a polytope, $P_i$,
that is the intersection of $H = O(d/\eps)^{(d-1)/2}$ many halfspaces in $\R^d$,
so that:
\begin{itemize}
\item[(i)] $P_i \subseteq L_f(y_i)$, and
\item[(ii)] $\vol(P_i) \geq \vol\left(L_f(y_i)\right) \cdot (1-\eps).$
\end{itemize}
This defines our function $g.$ It remains to prove that $\|f-g\|_1 = O(\eps)$.

We first point out that, by the definition of $g$, we have that $f(x) \geq g(x)$ for all $x \in \R^d.$
{This is because, if $g(x)=y_i$, it must be the case that $x\in P_i \subseteq L_f(y_i)$, by condition (i) above, and therefore $f(x)\geq y_i = g(x)$.}
So, to prove the lemma, it suffices to show that $\int_{\R^d} g(x)dx = 1 - O(\eps)$.
We start by noting that
$$1= \int_{\R^d} f(x) dx = \vol\left(\left\{ (x,y) \in \R^{d+1} \mid 0\leq y \leq f(x) \right\} \right) =\int_{\R_+} \vol \left(L_f(y)\right) dy.$$
Similarly, if we denote $L_g (y) = \{x \in \R^d \mid g(x) \geq y\}$, we have that
$$\int_{\R^d} g(x)dx = \int_{\R_+} \vol\left(L_g(y)\right)dy \;.$$
The following claim establishes that the contribution to $\int_{\R^d} f(x) dx$
from the points $x \in \R^d$ with $f(x) \leq y_{L-1}$ is small:
\begin{claim} \label{claim:conc}
It holds that $\int_{0}^{y_{L-1}} \vol\left(L_f(y)\right) dy \leq \eps.$
\end{claim}
\begin{proof}
We assume without loss of generality that $f$ attains its maximum value, $M_f$,
at $x=\mathbf{0}$. Let $R=L_f \left(\frac{M_f}{e}\right)$.
Notice that
$$1 =  \int_{\R_+} \vol \left(L_f(y)\right) dy \geq \int_{0 \leq y \leq M_f/e} \vol \left(L_f(y)\right) dy \geq  \int_{0 \leq y \leq M_f/e} \vol \left(R\right) dy
 = \frac{M_f}{e} \cdot  \vol \left(R\right) \;,$$
where we used the fact that $R \subseteq L_f(y)$ since $y \leq M_f/e.$
Hence, we have that
$$\vol(R) \leq e/M_f \;.$$
Moreover, we claim that, for $z \geq 1$, by the log-concavity of $f$
we have that
$$L_f(M_f e^{-z}) \subseteq zR \;.$$
Indeed, for $f(x)\geq M_f e^{-z}$, then $f(x/z) \geq f(0)^{(z-1)/z}f(x)^{1/z} \geq M_f/e$.
Therefore, we have that $x/z\in R$ or equivalently $x\in zR.$
Hence,
\begin{equation}\label{eqn:lc-vol-ub}
\vol\left(L_f(M_fe^{-z})\right)\leq O(z^d/M_f).
\end{equation}
Recall that, by our definition of $L$,
if we choose sufficiently large asymptotic constants,
it holds $y_{L-1}\leq \delta M_f$ for $\delta = \eps^2 / O(d)^{2d}.$
We now have the following sequence of inequalities:
\begin{align*}
\int_0^{y_{L-1}} \vol\left(L_f(y)\right)dy & = \\
&= \int_{\ln(M_f/y_{L-1})}^{\infty} \vol\left(L_f(M_fe^{-z})\right) M_f e^{-z} dz \tag{by the change of variable $y = M_f e^{-z}$} \\
& \leq \int_{\ln(1/\delta)}^\infty O(z^d e^{-z}) dz \tag{by (\ref{eqn:lc-vol-ub}) and the assumption $M_f/y_{L-1} \geq 1/\delta$}\\
& \leq \int_{\ln(1/\delta)}^\infty O(d)^d e^{-z/2}dz \tag{since $e^{z/2} \geq (z/2)^d/d!$} \\
& = O(d)^d \delta^{1/2} \\
& \leq \eps \;. \tag{using the definition of $\delta$}
\end{align*}
This completes the proof of Claim~\ref{claim:conc}.
\end{proof}

We now establish the following crucial claim:
\begin{claim}  \label{claim:volume-lb}
For $y_L \leq y \leq y_1$, we have that
$\vol\left(L_g(y)\right) \geq (1-\eps) \vol \left(L_f\left(\frac{y}{1-\eps}\right)\right).$
\end{claim}
\begin{proof}
Recall that $y_i \eqdef M_f (1-\eps)^i$, $i \in [L]$.
Since $y_1 > y_2 > \ldots > y_L,$ we can equivalently write (\ref{eqn:g}) as follows:
\begin{equation} \label{eqn:g2}
g(x) =\left\{
	\begin{array}{ll}
		 y_i,  \mbox{ where } i = \min\{j \in [L]: x \in P_j \}  & \mbox{if } x \in \cup_{j=1}^L P_j \\
		0 & \mbox{if } x \notin \cup_{j=1}^L P_j \;.
	\end{array}
\right.
\end{equation}
We claim that $L_g(y_i) = \bigcup_{1 \leq j \leq i} P_j$, $i \in [L]$.
Indeed, we can write
$$L_g(y_i) = \{ x \in \R^d \mid g(x) \geq y_i \} =   \bigcup_{1 \leq j \leq i}  \{ x \in \R^d \mid g(x) = y_j \}
=  \bigcup_{1 \leq j \leq i} \left(P_j \setminus \cup_{k<j} P_k\right)
=   \bigcup_{1 \leq j \leq i} P_j \;,$$
where the second and third equalities follow from (\ref{eqn:g2}).

For $y = y_1$, we thus have that
$$\vol(L_g(y_1)) = \vol(P_{1}) \geq (1-\eps) \vol(L_f(y_{1})) \geq  (1-\eps) \vol\left(L_f\left(\frac{y_1}{1-\eps}\right)\right) \;,$$
where the first inequality is implied by (ii), and the second inequality follows from
the fact $L_f(y) \supseteq L_f(y')$ whenever $y \leq y'$.

For $y_L \leq y < y_1$, consider the index $i \in [L-1]$ such that $y_{i+1} \leq y < y_{i} = \frac{y_{i+1}}{1-\eps}$.
By definition, we have that
$$L_g(y_{i}) \subseteq L_g(y) \subseteq L_g(y_{i+1}) \;.$$
Recalling that $L_g(y_i) = \bigcup_{1 \leq j \leq i} P_j$,
we obtain $L_g(y) \supseteq P_{i}$, and therefore
$$\vol(L_g(y)) \geq \vol(P_{i}) \geq (1-\eps) \vol(L_f(y_{i})) =  (1-\eps) \vol\left(L_f\left(\frac{y_{i+1}}{1-\eps}\right)\right)
\geq  (1-\eps) \vol\left(L_f\left(\frac{y}{1-\eps}\right)\right) \;,$$
where the second inequality is implied by (ii) and the third inequality uses the fact that $y \geq y_{i+1}$
and the fact $L_f(y) \supseteq L_f(y')$ whenever $y \leq y'$.
This completes the proof of Claim~\ref{claim:volume-lb}.
\end{proof}

We are now ready to complete the proof.
We have the following:
\begin{align*}
\int_{\R^d} g(x)dx
&= \int_{y_L}^{y_1} \vol \left( L_g (y)\right) dy \\
&\geq (1-\eps) \int_{y_L}^{y_1} \vol  \left(L_f \left(\frac{y}{1-\eps}\right)\right) dy \tag{by Claim~\ref{claim:volume-lb}} \\
&=  (1-\eps)^2 \cdot  \int_{y_L/(1-\eps)}^{M_f} \vol\left(L_f(y')\right)dy' \\
&=  (1-\eps)^2 \cdot \left( \int_{0}^{M_f} \vol\left(L_f(y)\right)dy  -  \int_{0}^{y_{L-1}} \vol\left(L_f(y)\right)dy \right) \\
& \geq   (1-\eps)^2 \cdot \left( \int_{\R^d} f(x)dx  - \eps \right)  \tag{by Claim~\ref{claim:conc}} \\
&=  (1-\eps)^2 \cdot (1-\eps) \\
&= 1-O(\eps) \;.
\end{align*}
The proof of Lemma~\ref{lem:approx} is now complete.
\end{proof}

We now proceed to define the family of subsets $\mathcal{A}$ and bound from above
its VC dimension. In particular, we define $\mathcal{A}$ to be the family of sets that exactly
express the differences between two elements of $\mathcal{C}_{d, \eps}$:
\begin{definition}
Define the family $\mathcal{A}_{d,\epsilon}$ of sets in $\R^d$ to be the collection of all sets of the form
$\{x\in \R^d: g(x) \geq g'(x)\}$
for some $g,g'\in \mathcal{C}_{d,\epsilon}$.
Notice that if $g, g' \in \mathcal{C}_{d,\epsilon}$ then
$\dtv(g,g') = \|g-g'\|_{\mathcal{A}_{d,\epsilon}} \;.$
\end{definition}

We show the following lemma:

\begin{lemma} \label{lem:family-A}
The VC dimension of $\mathcal{A}_{d,\epsilon}$ is at most $O(d/\eps)^{(d+1)/2}\log^2(1/\eps)$.
Furthermore, for $f, f'\in \mathcal{F}_d$, and $c>0$ is a sufficiently small constant,
we have that $\dtv(f,f')\leq \|f-f'\|_{\mathcal{A}_{d,c\eps}}+\eps/2$.
\end{lemma}
\begin{proof}
Note that a $g\in \mathcal{C}_{d,\epsilon}$ is determined completely
by $L=O((d/\eps)\log(d/\eps))$ values $y_i$ and
$LH = O(d/\eps)^{(d+1)/2}\log(1/\eps)$ halfspaces used to define
the $L$ convex polytopes $P_i$.
{We will show} that if $g' \in \mathcal{C}_{d,\epsilon}$ is
defined by $L$ values $y'_i$ and another set of $LH$ halfspaces,
and if $x\in \R^d$, then it is possible to determine whether or not $g(x)\geq g'(x)$ based solely on:
\begin{itemize}
\item The relative ordering of the $y_i$ and $y'_i$.
\item Which of the $2LH$ halfspaces $x$ belongs to.
\end{itemize}
Now consider an arbitrary set $T$ of $n$ points in $\R^d$.
We wish to bound the number of possible distinct sets
that can be obtained by the intersection of $T$
with a set in $\mathcal{A}_{d,\eps}$.
By the above, the intersection will be determined by:
\begin{itemize}
\item The relative ordering of the $2L$ elements given by the $y_i$ and $y'_i$.
\item The intersections of each of the $2LH$ halfspaces defining $g$ and $g'$ with $T$.
\end{itemize}
Note that the number of orderings in question is at most $(2L)!$.
{
Formally, we can write $P_i = \bigcap_{j=1}^H P_{i,j}$ for halfspaces $P_{i,j}$,
and similarly $P'_i = \bigcap_{j=1}^H P'_{i,j}$,
where $P_i$ and $P'_i$ appear in the definition of $g$ and $g'$ respectively.
We have the following:
\begin{claim}  \label{claim:set-functions}
There exist at most $(2L)!$ different $2L$-ary set functions $F_k$
such that for any $g,g' \in \mathcal{C}_{d,\epsilon}$
the set $\{x: g(x) \geq g'(x)\}$ is given by
$F_k(P_{1,1}, \dots, P_{L,H}, P'_{1,1}, \dots, P'_{L,H})$ for some $k$.
Furthermore, these functions are distributive over intersection,
i.e., for all $k$ and $T, S_1,\dots, S_{2LH} \subseteq \R^d$,
we have that $F_k(S_1,\dots, S_{2LH}) \cap T= F_k(S_1 \cap T, \dots, S_{2LH} \cap T)$.
\end{claim}
\begin{proof}
Note that for a given $x \in \R^d$, we have that $g(x) \geq g'(x)$
if and only if there is an $i$ such that $x \in P_i$
and for all $i'$ with $y'_{i'} \geq y_i$
we have $x \notin P'_{i'}$.
That is,
$\{x : g(x) \geq g'(x)\} = \bigcup_{i=1}^L \left( P_i \setminus \bigcup_{i': y'_{i'} > y_i} P'_{i'} \right)$.
In terms of halfspaces, this can be equivalently written as follows:
$$\{x : g(x) \geq g'(x)\} = \bigcup_{i=1}^L \left( \bigcap_{j=1}^H P_{i,j}  \setminus \bigcup_{i': y'_{i'} > y_i} \bigcap_{j=1}^H P'_{i',j} \right) \;.$$
Note that, viewed as a function of the halfspaces, the above expression
only depends on the relative ordering of the $y_i$ and $y'_i$.
Thus, we can express this as one of at most $(2L)!$ functions of these halfspaces.

Since these functions are defined using only unions,
intersections and differences (which all distribute over intersections),
so do the $F_k$.
\end{proof}

It is well-known that for any halfspace  the number of possible intersections
with a set $T$ of size $n$ is at most $O(n)^d$.
By Claim~\ref{claim:set-functions}, for any $A \in \mathcal{A}_{d,\eps}$
we have that $A \cap T = F_k(S_1 \cap T,\dots, S_{2LH} \cap T)$ for halfspaces
$S_1, \dots , S_{2LH}$.
There are $(O(n)^d)^{2LH}$ different $2LH$-tuples of intersections of halfspaces with $T$
and at most $(2L)!$ different $F_k$.
} Therefore, the number of possible intersections of an element of $\mathcal{A}_{d,\eps}$ with $T$ is at most
\begin{equation} \label{eqn:intersections}
(2L)!O(n)^{2dLH} = \exp(O(d/\eps)^{(d+1)/2}\log(1/\eps)\log(n)) \;.
\end{equation}
On the other hand, if $\mathcal{A}_{d,\eps}$ has VC dimension $n$, (\ref{eqn:intersections})
must be at least $2^n$. Therefore, if $n$ is the VC-dimension of $\mathcal{A}_{d,\eps}$, we have that
$$n/\log(n) = O(d/\eps)^{(d+1)/2}\log(1/\eps) \;,$$
and therefore,
$$n = O(d/\eps)^{(d+1)/2}\log^2(1/\eps) \;.$$
For the claim comparing the variation distance to $\| \cdot \|_{\mathcal{A}_{d,c\eps}}$, 
note that, by Lemma \ref{lem:approx}, if $c$ is chosen to be sufficiently small,
there exist $g,g'\in\mathcal{C}_{d,c\eps}$
so that $\dtv(f,g),\dtv(f',g') \leq \eps/8.$ We then have that
\begin{align*}
\dtv(f,f') &\leq \dtv(f,g)+\dtv(f',g')+\dtv(g,g')\\
&\leq \eps/4 + \|g-g'\|_{\mathcal{A}_{d,c\eps}}\\
& \leq \eps/4+\|f-f'\|_{\mathcal{A}_{d,c\eps}}+\dtv(f,g)+\dtv(f',g') \\
& \leq \|f-f'\|_{\mathcal{A}_{d,c\eps}}+\eps/2 \;.
\end{align*}
This completes the proof of Lemma~\ref{lem:family-A}.
\end{proof}

\noindent The proof of Proposition~\ref{prop:main} and Theorem~\ref{thm:main} is now complete.

\section{Conclusions} \label{sec:concl}
In this paper, we gave the first sample complexity upper bound for learning log-concave
densities on $\R^d$. Our upper bound agrees with the previously known lower bound up to a multiplicative factor of
$\tilde O_d (\eps^{-2})$.
No sample complexity upper bound was previously known for this problem for any $d>3$.

Our result is a step towards understanding the learnability of log-concave densities
in multiple dimensions.
A number of interesting open problems remain. 
We outline the two immediate ones here:

\begin{itemize}
\item What is the {\em optimal} sample complexity of log-concave density estimation?
It is a plausible conjecture that the correct answer, under the total variation distance, 
is $\Theta_d \left( (1/\eps)^{d/2+2} \right)$. We believe that a more sophisticated
version of our structural approximation results could give such an upper bound. 
On the other hand, it seems likely that 
an adaptation of the construction in~\cite{KimSam16} could yield a matching lower bound.

\item Is there a {\em polynomial time algorithm}
(as a function of the sample complexity) to learn log-concave densities on $\R^d$?
The estimator underlying this work (Lemma~\ref{lem:generic-vc}) 
has been previously exploited~\cite{CDSS13, CDSS14, ADLS17} 
to obtain computationally efficient learning algorithms for $d=1$ -- 
in fact, running in sample near-linear time~\cite{ADLS17}.
Obtaining a computationally efficient algorithm for the case of
general dimension is a challenging and important open question.
\end{itemize}


\bibliographystyle{alpha}
\bibliography{allrefs}

\end{document}